\documentclass{article}

\usepackage{neurips_2024}

\usepackage{microtype}
\usepackage{graphicx}
\usepackage{booktabs}
\usepackage{amsmath,amssymb,amsthm}
\usepackage{mathtools}
\usepackage{algorithm}
\usepackage{algpseudocode}
\usepackage{xcolor}
\usepackage{pifont}
\usepackage{enumitem}
\usepackage{subcaption}
\usepackage{hyperref}
\usepackage{cleveref}
\usepackage{pgfplots}
\pgfplotsset{compat=1.18}
\usepgfplotslibrary{groupplots,fillbetween}
\usetikzlibrary{shapes,arrows.meta,positioning,fit,backgrounds,calc,patterns,decorations.pathreplacing}

\newtheorem{theorem}{Theorem}[section]

\newtheorem{definition}[theorem]{Definition}

\newtheorem{example}[theorem]{Example}

\newcommand{\sys}{MoLoRA\xspace}
\newcommand{\RR}{\mathbb{R}}
\newcommand{\ZZ}{\mathbb{Z}}

\newcommand{\cO}{\mathcal{O}}
\newcommand{\cL}{\mathcal{L}}
\newcommand{\cR}{\mathcal{R}}
\newcommand{\cA}{\mathcal{A}}
\newcommand{\cM}{\mathcal{M}}
\newcommand{\cT}{\mathcal{T}}
\newcommand{\cX}{\mathcal{X}}
\newcommand{\cC}{\mathcal{C}}
\usepackage{xspace}

\title{MoLoRA: Composable Specialization via Per-Token Adapter Routing}

\author{
  Shrey Shah\\
  \texttt{shreyshah@microsoft.com}
  \And
  Justin Wagle\\
  \texttt{justiwag@microsoft.com}
}
\date{}

\begin{document}

\maketitle

\begin{abstract}
Multi-adapter serving systems route entire sequences to a single adapter, forcing a choice when requests span multiple domains. This assumption fails in two important settings: (1) multimodal generation, where text and image tokens require different adapters within the same sequence, and (2) mixed-capability requests like ``write code to solve this equation,'' which need expertise from multiple specialized adapters.
We introduce \emph{per-token routing}, which routes individual tokens to adapters based on either vocabulary structure (for multimodal models) or learned gating (for semantic specialization). Per-token routing is provably optimal, achieving work $N$ for $N$ tokens versus $K \cdot N$ for per-sequence routing with $K$ adapter types.
Our key contribution is \textbf{MoLoRA} (Mixture of LoRA), which enables \emph{composable specialization}: load multiple domain-specific adapters and let a learned router select the appropriate adapter per-token. We demonstrate that \textbf{specialization dramatically beats scale}: MoLoRA enables \textbf{Qwen3-1.7B to exceed Qwen3-8B} across four reasoning benchmarks while being \textbf{4.7$\times$ smaller}. This enables modular expertise at inference time: train focused LoRAs independently, combine them without retraining, and add new capabilities by simply loading new adapters.
\end{abstract}

\section{Introduction}
\label{sec:intro}

Low-rank adaptation (LoRA)~\citep{hu2021lora} enables efficient fine-tuning of large language models by learning low-rank updates to pretrained weights. Multi-adapter serving systems such as S-LoRA~\citep{sheng2024slora} and Punica~\citep{chen2023punica} extend this to serve multiple adapters concurrently, enabling a single base model deployment to serve many fine-tuned variants.

However, existing systems share a fundamental assumption: each request routes to exactly one adapter. Formally, given a batch of $N$ tokens from $B$ sequences, these systems compute the routing function $j: [B] \to [K]$ mapping sequences to adapters, then apply:
\begin{equation}
    h_i = x_i W + x_i A_{j(s_i)} B_{j(s_i)}
    \label{eq:perseq}
\end{equation}
where $s_i$ denotes the sequence containing token $i$. This \emph{per-sequence} routing causes two problems:

\paragraph{Problem 1: Multimodal Efficiency.} Frontier models like Gemini~\citep{gemini2025} generate interleaved text and images in a single response---an illustrated recipe where text instructions alternate with generated images. With per-sequence routing, all tokens must use the same adapter, even though text and image tokens should use modality-specialized adapters. This forces either suboptimal adapter selection or expensive sequence splitting ($K$ forward passes for $K$ modalities).

\paragraph{Problem 2: Mixed-Capability Quality.} Consider a request like ``write Python code to solve this differential equation.'' A code adapter excels at syntax but not calculus; a math adapter handles equations but not programming idioms. With per-sequence routing, we must choose one and accept suboptimal quality on the other capability. No single adapter can match multiple specialists.

We introduce \emph{per-token routing}, which generalizes Equation~\ref{eq:perseq} to route individual tokens:
\begin{equation}
    h_i = x_i W + x_i A_{r(i)} B_{r(i)}
    \label{eq:pertok}
\end{equation}
where $r: [N] \to [C]$ routes each token to one of $C$ computational targets. Per-token routing solves both problems: it reduces $K$ forward passes to 1 (Problem 1), and enables different tokens to use different specialized adapters within the same sequence (Problem 2). The routing function $r$ can be deterministic (based on vocabulary structure, for multimodal models) or learned (for semantic specialization).

\begin{figure}[t]
\centering
\begin{tikzpicture}[
    token/.style={rectangle, draw, minimum width=0.55cm, minimum height=0.5cm, font=\tiny},
    texttoken/.style={token, fill=blue!20},
    imgtoken/.style={token, fill=orange!30},
    audiotoken/.style={token, fill=green!20},
    videotoken/.style={token, fill=purple!20},
    adapter/.style={rectangle, draw, rounded corners, minimum width=1.0cm, minimum height=0.6cm, font=\scriptsize},
    arrow/.style={-Stealth, thick},
    label/.style={font=\scriptsize}
]

\node[font=\small\bfseries] at (-3.2, 2.2) {Per-Sequence Routing};

\node[texttoken] (t1) at (-4.8, 1.2) {T};
\node[imgtoken] (i1) at (-4.2, 1.2) {I};
\node[audiotoken] (a1) at (-3.6, 1.2) {A};
\node[videotoken] (v1) at (-3.0, 1.2) {V};
\node[texttoken] (t2) at (-2.4, 1.2) {T};
\node[imgtoken] (i2) at (-1.8, 1.2) {I};

\node[label, anchor=east] at (-5.1, 1.2) {Input:};

\node[adapter, fill=gray!20] (adapt1) at (-3.3, 0) {Adapter 1};

\draw[arrow, blue!60] (t1.south) -- (adapt1.north);
\draw[arrow, orange!60] (i1.south) -- (adapt1.north);
\draw[arrow, green!60] (a1.south) -- (adapt1.north);
\draw[arrow, purple!60] (v1.south) -- (adapt1.north);
\draw[arrow, blue!60] (t2.south) -- (adapt1.north);
\draw[arrow, orange!60] (i2.south) -- (adapt1.north);

\node[font=\tiny, red!70!black, align=center] at (-3.3, -0.8) {All tokens $\to$ same adapter\\(suboptimal for mixed content)};

\node[font=\small\bfseries] at (3.2, 2.2) {Per-Token Routing (Ours)};

\node[texttoken] (pt1) at (1.2, 1.2) {T};
\node[imgtoken] (pi1) at (1.8, 1.2) {I};
\node[audiotoken] (pa1) at (2.4, 1.2) {A};
\node[videotoken] (pv1) at (3.0, 1.2) {V};
\node[texttoken] (pt2) at (3.6, 1.2) {T};
\node[imgtoken] (pi2) at (4.2, 1.2) {I};

\node[label, anchor=east] at (0.9, 1.2) {Input:};

\node[adapter, fill=blue!20] (tadapt) at (1.2, 0) {Text};
\node[adapter, fill=orange!30] (iadapt) at (2.4, 0) {Image};
\node[adapter, fill=green!20] (aadapt) at (3.6, 0) {Audio};
\node[adapter, fill=purple!20] (vadapt) at (4.8, 0) {Video};

\draw[arrow, blue!60] (pt1.south) -- ([xshift=-2pt]tadapt.north);
\draw[arrow, orange!60] (pi1.south) -- ([xshift=-2pt]iadapt.north);
\draw[arrow, green!60] (pa1.south) -- (aadapt.north);
\draw[arrow, purple!60] (pv1.south) -- (vadapt.north);
\draw[arrow, blue!60] (pt2.south) to[out=-90, in=60] ([xshift=2pt]tadapt.north);
\draw[arrow, orange!60] (pi2.south) to[out=-90, in=60] ([xshift=2pt]iadapt.north);

\node[font=\tiny, green!50!black, align=center] at (3.0, -0.8) {Each token $\to$ specialized adapter\\(optimal for multimodal)};

\draw[dashed, gray] (0, 2.5) -- (0, -1.2);

\end{tikzpicture}
\caption{Per-sequence routing (left) sends all tokens in a sequence to the same adapter, even when tokens have different modalities (T=text, I=image, A=audio, V=video). Per-token routing (right) routes each token to its modality-specialized adapter within a single forward pass.}
\label{fig:routing_comparison}
\end{figure}
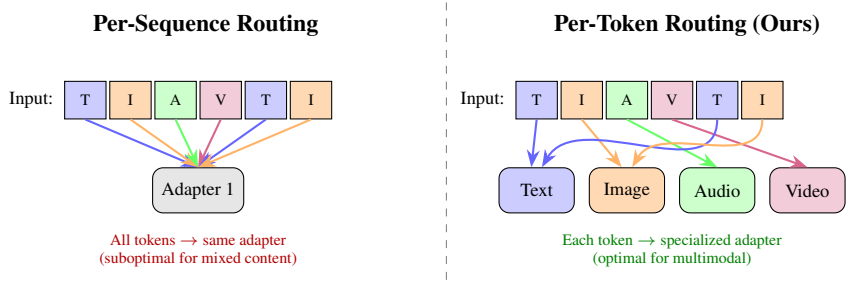

\paragraph{Contributions.} We make three contributions:

\begin{enumerate}
    \item \textbf{Per-token routing framework} (\S\ref{sec:routing}). We formalize per-token routing and prove it is computationally optimal: work $N$ for $N$ tokens versus $K \cdot N$ for per-sequence routing with $K$ adapter types.

    \item \textbf{MoLoRA: Composable specialization} (\S\ref{sec:molora}). We introduce MoLoRA (Mixture of LoRA), which extends per-token routing with learned gating. MoLoRA enables \emph{composable specialization}: load multiple domain-specific adapters and let a learned router select per-token. We demonstrate that \textbf{specialization beats scale}: Qwen3-1.7B + MoLoRA exceeds Qwen3-8B (4.7$\times$ larger) on all four reasoning benchmarks---GSM8K (+14\%), MATH (+8\%), BBH (+2.5\%), GPQA (+2.1\%).

    \item \textbf{Systems and empirical validation} (\S\ref{sec:architecture}--\ref{sec:evaluation}). A hot-set memory architecture enables CUDA graph capture, reducing P99 latency by 67$\times$. Per-token routing achieves $K\times$ improvement for $K$-modality workloads: 4.1$\times$ from pass reduction, compounding to 5.5$\times$ with systems optimizations.
\end{enumerate}

\section{Background and Related Work}
\label{sec:background}

\subsection{Low-Rank Adaptation}

LoRA~\citep{hu2021lora} fine-tunes pretrained weights $W \in \RR^{d \times d}$ via low-rank updates $\Delta W = AB$ where $A \in \RR^{d \times r}$, $B \in \RR^{r \times d}$, and $r \ll d$. The forward pass computes $h = xW + xAB$, adding only $2dr$ parameters per adapted layer. Extensions include QLoRA~\citep{dettmers2023qlora} for quantized training and DoRA~\citep{liu2024dora} for weight decomposition.

\subsection{Multi-Adapter Serving}

S-LoRA~\citep{sheng2024slora} introduced unified paging for adapter weights and KV cache, enabling adapter catalogs that exceed GPU memory. S-LoRA and Punica~\citep{chen2023punica} use specialized kernels optimized for the memory-bound regime of small-batch LoRA computation, achieving high throughput through unified memory pools and LRU eviction. LoRAServe~\citep{jaiswal2025loraserve} addresses rank heterogeneity across adapters via dynamic placement and GPU Direct RDMA, but like prior systems assumes per-sequence routing: $j: [B] \to [K]$ maps sequences to adapters, precluding per-token differentiation.

Orthogonally, work on LoRA \emph{composition} addresses blending adapter effects rather than routing requests: MoLE~\citep{wu2024mixtureloraexperts} learns per-layer gating to weight LoRA outputs, while Multi-LoRA Composition~\citep{zhong2024multilora} proposes LoRA Switch and LoRA Composite for image generation. Recent work explores per-token adapter selection for \emph{model quality}: \citet{feng2023tokenlevel} use cosine similarity to route tokens to adapters for improved generalization, and LLaVA-MoLE~\citep{chen2024llavamole} applies MoE-style routing over LoRA experts during multimodal instruction finetuning. These focus on \emph{training} or \emph{generalization}---our work instead targets \emph{serving efficiency}, introducing systems optimizations (CUDA graphs, hot-set memory) absent from prior per-token approaches.

\subsection{Mixture of Experts}

MoE architectures~\citep{shazeer2017outrageously,fedus2022switch,lepikhin2020gshard} route tokens to specialized expert networks via learned gating functions. Existing multi-adapter systems route at the sequence level, making them fundamentally different from MoE: one routing decision per sequence versus one per token. Per-token routing bridges this gap---it makes adapter dispatch identical to MoE dispatch at the infrastructure level:

\begin{center}
\begin{tabular}{lll}
\toprule
\textbf{Component} & \textbf{MoE} & \textbf{Per-Token Adapters} \\
\midrule
Routing granularity & Per-token & Per-token \\
Routing function & Learned $g_\theta(x)$ & Deterministic $\cR_\text{vocab}(v)$ \\
Post-routing & Histogram + scatter-gather & Histogram + scatter-gather \\
\bottomrule
\end{tabular}
\end{center}

This structural equivalence has concrete benefits. We adopt \emph{adaptive tiling} from MoE systems~\citep{gale2023megablocks}: the histogram of tokens per target determines tile sizes in the compute kernel, with small groups using smaller tiles and large groups using larger tiles for better memory bandwidth utilization. More broadly, the dispatch kernel is target-agnostic---optimizations transfer bidirectionally between MoE and multi-adapter systems.

\subsection{Multimodal LLMs and Vocabulary Structure}

Unified multimodal architectures such as Chameleon~\citep{team2024chameleon} process multiple modalities through shared transformer backbones with unified vocabularies. These models encode modality information directly in token indices: text tokens occupy one vocabulary range, image tokens another, and so forth. We exploit this structure for per-token routing: a comparison against vocabulary boundaries determines the appropriate adapter for each token.

\section{Per-Token Routing Framework}
\label{sec:routing}

We develop a framework for per-token adapter routing, establishing its relationship to per-sequence routing and mixture-of-experts.

\subsection{From Sequences to Tokens}

Let $\cX = \{x_1, \ldots, x_N\}$ denote a batch of $N$ tokens with associated vocabulary indices $v_i \in [V]$ and sequence memberships $s_i \in [B]$. Per-sequence routing computes $r(i) = j(s_i)$ where $j: [B] \to [K]$ assigns adapters to sequences. This constrains all tokens in a sequence to use the same adapter.

Per-token routing removes this constraint by defining $r: [N] \to [C]$ directly on tokens. The key insight is that multimodal vocabularies encode routing information:

\begin{definition}[Vocabulary Routing]
\label{def:vocab_routing}
Given vocabulary breaks $\mathbf{b} = (b_0, b_1, \ldots, b_M)$ with $b_0 = 0$ and $b_M = V$, the vocabulary routing function is:
\begin{equation}
    \cR_\text{vocab}(v) = m \iff v \in [b_{m-1}, b_m)
\end{equation}
\end{definition}

For a multimodal model with $M$ modalities~\citep{team2024chameleon}, tokens route to adapters based on their vocabulary range: text tokens to a text adapter, image tokens to an image adapter, audio to audio, video to video, and so forth. The routing decision requires only $M-1$ integer comparisons.

\paragraph{Assumption.} This formulation assumes modalities occupy contiguous vocabulary ranges, as in unified multimodal tokenizers. Non-contiguous assignments require an $O(V)$ lookup table; we focus on the contiguous case which covers Chameleon-family models and similar architectures.

\subsection{Compositional Routing}

Per-token routing enables routing over \emph{product spaces}, providing exponential expressiveness from independently trained components.

\begin{definition}[Compositional Routing]
\label{def:compositional}
Given adapter set $\cA$ and modality set $\cM$, compositional routing defines targets $\cC = \cA \times \cM$ with routing function:
\begin{equation}
    r(i) = (a_i, m_i) \in \cA \times \cM
\end{equation}
where $a_i$ is determined by request metadata (e.g., customer identity) and $m_i = \cR_\text{vocab}(v_i)$ is determined by vocabulary structure.
\end{definition}

The composite target $c = a \cdot |\cM| + m$ indexes into a weight tensor of shape $(|\cA|, |\cM|, d, r)$, providing $|\cA| \cdot |\cM|$ distinct computational paths with a single dispatch operation. The key advantage is \emph{unified serving}: rather than deploying $|\cA| \cdot |\cM|$ separate model instances, a single deployment handles all combinations through compositional indexing.

\begin{example}
With 64 customer adapters and 4 modalities, compositional routing provides 256 computational paths from a single $(64, 4, d, r)$ weight tensor, served through unified infrastructure rather than 256 separate deployments.
\end{example}

The framework extends naturally to additional factors. Token-level specialization $\cT$ (e.g., routing special tokens differently) yields product space $\cA \times \cM \times \cT$.

\subsection{Unified Dispatch Infrastructure}
\label{sec:dispatch}

The dispatch kernel that routes tokens to targets in a LoRA setting is equivalent to that of routing tokens to MoE experts. Given routing decisions $r \in [C]^N$, the kernel:

\begin{enumerate}
    \item Constructs a histogram $h \in \ZZ^C$ counting tokens per target
    \item Allocates positions via atomic increment: $\text{pos} \gets \text{atomicAdd}(h[r[t]], 1)$
    \item Emits pointer arrays $\mathbf{xs}[c, \text{pos}], \mathbf{ys}[c, \text{pos}]$ for scatter-gather
\end{enumerate}

This three-step pattern is target-agnostic: the kernel operates on abstract routing decisions without knowledge of whether targets represent experts, adapters, or compositional combinations thereof. The histogram enables adaptive tiling in subsequent compute kernels---targets with few tokens use smaller tiles while targets with many tokens use larger tiles for better memory bandwidth utilization (Appendix~\ref{app:dispatch}).

For compositional routing, the kernel computes composite targets by combining adapter indices from request metadata with modality indices from vocabulary structure (Algorithm~\ref{alg:compositional_gate}).

\begin{algorithm}[t]
\caption{Compositional Gating Kernel}
\label{alg:compositional_gate}
\begin{algorithmic}[1]
\Require Token features $X$, vocab indices $V$, adapter indices $A$, modality breaks $\mathbf{b}$
\Ensure Histogram $h \in \ZZ^{|\cA| \times |\cM|}$, row pointers $\mathbf{xs}, \mathbf{ys}$
\State $h \gets \mathbf{0}_{|\cA| \times |\cM|}$
\For{token $t = 1, \ldots, N$ \textbf{in parallel}}
    \State $a \gets A[t]$ \Comment{Adapter from request metadata}
    \State $m \gets \text{FindModality}(V[t], \mathbf{b})$ \Comment{Modality from vocabulary}
    \State $c \gets a \cdot |\cM| + m$ \Comment{Composite target}
    \State $\text{pos} \gets \text{atomicAdd}(h[c], 1)$
    \State $\mathbf{xs}[c, \text{pos}] \gets \&X[t]$; $\mathbf{ys}[c, \text{pos}] \gets \&Y[t]$
\EndFor
\end{algorithmic}
\end{algorithm}

\subsection{Theoretical Analysis}

\begin{theorem}[Routing Complexity]
\label{thm:complexity}
Vocabulary routing achieves $\cO(1)$ per-token routing cost for fixed $M$, compared to $\cO(E \cdot d)$ for learned MoE gating over $E$ experts with hidden dimension $d$.
\end{theorem}

Vocabulary routing requires $M-1$ integer comparisons, which are unrolled at compile time for typical $M \leq 4$. Learned gating computes $\text{softmax}(Wx + b)$ where $W \in \RR^{E \times d}$, requiring $E \cdot d$ multiplications. For $d = 4096$ and $E = 8$, this represents a $\sim$10,000$\times$ difference in routing overhead.

\begin{theorem}[Expressiveness]
\label{thm:expressiveness}
When the optimal routing function $\cR^*$ satisfies $\cR^*(x) = m$ iff $v(x) \in [b_{m-1}, b_m)$, vocabulary routing achieves $\cR_\text{vocab} = \cR^*$.
\end{theorem}

When vocabulary structure determines optimal routing---as in multimodal models where modality determines the appropriate adapter---deterministic routing matches learned routing without parameter overhead.

\subsection{Key Theoretical Results}

We establish two results that distinguish per-token from per-sequence routing: computational optimality and a unification with sparse attention.

\begin{theorem}[Computational Optimality]
\label{thm:optimality}
Consider a batch of $N$ tokens with arbitrary modality assignments processed by $K$ adapters. Any correct routing strategy requires work $W$ satisfying:
\begin{align}
    W_\text{per-seq} &\geq K \cdot N \cdot c_\text{pass} \quad \text{(per-sequence)} \\
    W_\text{per-tok} &= N \cdot c_\text{pass} \quad \text{(per-token)}
\end{align}
where $c_\text{pass}$ is the cost of processing one token through one adapter. Per-token routing achieves the minimum work $N \cdot c_\text{pass}$.
\end{theorem}

\begin{theorem}[Sparse Attention Equivalence]
\label{thm:sparse_attention}
Per-token adapter routing with $K$ adapters and vocabulary routing function $\cR_\text{vocab}$ is equivalent to sparse attention with a block-diagonal attention pattern, where the vocabulary partitioning determines the block structure.
\end{theorem}

\begin{proof}[Proof sketch]
Define attention weights $\alpha_{ij} = \mathbf{1}[\cR_\text{vocab}(v_i) = \cR_\text{vocab}(v_j)]$. The resulting attention pattern is block-diagonal with blocks corresponding to modality groups. The adapter computation $x_i A_{m(i)} B_{m(i)}$ can be written as:
\begin{equation}
    \sum_{j} \alpha_{ij} \cdot x_j A_{\cR_\text{vocab}(v_j)} B_{\cR_\text{vocab}(v_j)}
\end{equation}
which is a sparse attention operation where each token attends only to tokens of the same modality, using modality-specific projections.
\end{proof}

\paragraph{Implications.} This equivalence unifies per-token adapter routing with the broader sparse attention literature. Optimizations developed for sparse attention (e.g., block-sparse patterns, hardware-efficient implementations) directly apply to adapter routing, and vice versa. The block-diagonal structure enables efficient parallel execution: each modality group can be processed independently, achieving perfect load balancing when group sizes are similar.

\section{MoLoRA}
\label{sec:molora}

We now present \textbf{MoLoRA (Mixture of LoRA)}, which extends per-token routing with learned gating to enable \emph{composable specialization}---the ability to load multiple specialized LoRA adapters simultaneously and route tokens dynamically based on content.

\subsection{Motivation}

Traditional adapter serving requires choosing a single adapter per request: a math LoRA for mathematical reasoning, a code LoRA for programming, a creative LoRA for writing. However, real-world requests often require multiple capabilities. A request like ``write Python code to solve this differential equation'' needs both mathematical and programming expertise. Per-sequence routing forces a choice, accepting suboptimal quality on one capability.

Per-token routing with learned gating eliminates this trade-off. By loading multiple specialized adapters and routing per-token, a single serving endpoint achieves the quality benefits of all fine-tunes combined:
\begin{itemize}[nosep,leftmargin=*]
    \item \textbf{Single-domain tasks}: The router selects the specialized adapter, matching single-adapter quality
    \item \textbf{Mixed-capability tasks}: The router selects different adapters for different tokens, combining expertise within a single sequence
\end{itemize}

This enables \emph{modular expertise at inference time}: train focused LoRAs independently, combine them without retraining, and add new capabilities by loading new adapters.

\subsection{Limitations of Vocabulary Routing}
\label{sec:molora_motivation}

Deterministic vocabulary routing requires modality information encoded in token indices. We identify four scenarios where this assumption fails, motivating learned routing.

\paragraph{Scenario 1: Encoder-Based Multimodal Models.}
Models like LLaVA~\citep{liu2023visual}, Flamingo~\citep{alayrac2022flamingo}, and Qwen-VL~\citep{bai2023qwen} process images through separate encoders (e.g., CLIP~\citep{radford2021learning}) before projecting into the LLM's embedding space. The resulting ``image tokens'' occupy the \emph{same vocabulary range} as text tokens. Unlike Chameleon's disjoint ranges, these models provide no vocabulary-level signal distinguishing modalities.

\paragraph{Scenario 2: Semantic Specialization.}
Consider adapters specialized for code, mathematical reasoning, creative writing, and technical documentation. These domains share the same vocabulary---the word ``function'' appears in all four contexts with identical token IDs. Yet optimal adaptation differs: code adapters should emphasize syntax patterns, math adapters logical structure. Vocabulary-based routing cannot distinguish these cases.

\paragraph{Scenario 3: Sub-Modality Granularity.}
Even when vocabulary encodes modality, finer-grained specialization may be valuable. Within image tokens, photographs, diagrams, and charts may benefit from different adapters. Vocabulary routing provides only coarse modality-level grouping; learned routing enables arbitrary granularity.

\paragraph{Scenario 4: Compositional Multi-Attribute Routing.}
Production settings often require routing along multiple dimensions: modality $\times$ domain $\times$ task. Vocabulary routing handles only one dimension. Learned routing naturally extends to product spaces by predicting multi-dimensional routing targets.

\subsection{Router Architecture}

Given input $x \in \RR^{B \times L \times d}$, the router $g_\theta: \RR^d \to \RR^K$ produces adapter logits per token. We apply top-$k$ selection followed by softmax:
\begin{equation}
    w_{i} = \text{softmax}(\text{TopK}(g_\theta(x_i), k)), \quad
    \Delta h_i = \sum_{j \in \text{TopK}} w_{i,j} \cdot x_i A^{(j)} B^{(j)}
\end{equation}

The router is a 2-layer MLP with hidden dimension 64 and GELU activation:
\begin{equation}
    g_\theta(x) = W_2 \cdot \text{GELU}(W_1 x + b_1) + b_2
\end{equation}
This adds minimal parameters ($64d + 64K$) while enabling input-dependent adapter selection. Following Switch Transformer~\citep{fedus2022switch}, we add an auxiliary load-balancing loss to encourage uniform adapter utilization.

\begin{figure}[t]
\centering
\begin{tikzpicture}[
    box/.style={rectangle, draw, rounded corners, minimum width=1.4cm, minimum height=0.5cm, font=\scriptsize, align=center},
    token/.style={rectangle, draw, minimum width=0.6cm, minimum height=0.5cm, font=\scriptsize},
    adapter/.style={rectangle, draw, minimum width=0.9cm, minimum height=0.5cm, font=\scriptsize, align=center},
    router/.style={rectangle, draw, fill=yellow!20, rounded corners, minimum width=1.4cm, minimum height=0.5cm, font=\scriptsize},
    myarrow/.style={-Stealth, thick},
    grayarrow/.style={-Stealth, gray, thick},
]

\node[font=\small\bfseries] at (-4.7, 3.2) {Logical View (Per-Token)};

\node[token, fill=blue!25] (tok) at (-4.7, 2.3) {$x_i$};

\node[router] (router) at (-4.7, 1.3) {Router $g_\theta$};
\draw[myarrow] (tok.south) -- (router.north);

\node[box, fill=orange!10] (topk) at (-4.7, 0.3) {Top-$k$};
\draw[myarrow] (router.south) -- (topk.north);

\node[adapter, fill=blue!20] (a1) at (-6.35, -1.0) {$A_1$};
\node[adapter, fill=gray!15] (a2) at (-5.35, -1.0) {$A_2$};
\node[adapter, fill=blue!20] (a3) at (-4.05, -1.0) {$A_3$};
\node[adapter, fill=gray!15] (a4) at (-3.05, -1.0) {$A_4$};

\draw[myarrow, blue!70, thick] (topk.south) -- node[midway, sloped, font=\tiny, fill=white, inner sep=2pt] {$w_1$} (a1.north);
\draw[myarrow, blue!70, thick] (topk.south) -- node[midway, sloped, font=\tiny, fill=white, inner sep=2pt] {$w_3$} (a3.north);

\draw[gray, dashed] (topk.south) -- node[pos=1.0, font=\scriptsize, fill=white, inner sep=0.5pt] {$\times$} ($(topk.south)!0.59!(a2.north)$);
\draw[gray, dashed] (topk.south) -- node[pos=1.0, font=\scriptsize, fill=white, inner sep=0.5pt] {$\times$} ($(topk.south)!0.59!(a4.north)$);

\node[box, fill=green!15] (sum) at (-4.7, -2.1) {$\sum w_j$};
\draw[myarrow, blue!70] (a1.south) -- (sum.north);
\draw[myarrow, blue!70] (a3.south) -- (sum.north);

\node[box, fill=gray!20] (out) at (-4.7, -3.1) {$\Delta h_i$};
\draw[myarrow] (sum.south) -- (out.north);

\node[font=\tiny, align=center] at (-4.7, -3.9) {$\Delta h_i = \sum_{j \in \text{TopK}} w_{ij} \cdot x_i A^{(j)} B^{(j)}$};

\draw[dashed, gray, thick] (-1.5, 3.5) -- (-1.5, -4.2);

\node[font=\small\bfseries] at (3.3, 3.2) {Physical View (Grouped Dispatch)};

\node[token, fill=blue!25] (t1) at (0.8, 2.3) {$x_1$};
\node[token, fill=green!25] (t2) at (1.8, 2.3) {$x_2$};
\node[token, fill=blue!25] (t3) at (2.8, 2.3) {$x_3$};
\node[token, fill=orange!25] (t4) at (3.8, 2.3) {$x_4$};
\node[token, fill=blue!25] (t5) at (4.8, 2.3) {$x_5$};
\node[token, fill=green!25] (t6) at (5.8, 2.3) {$x_6$};

\node[font=\tiny, blue!70] at (0.8, 2.8) {$\rightarrow$1};
\node[font=\tiny, green!60!black] at (1.8, 2.8) {$\rightarrow$2};
\node[font=\tiny, blue!70] at (2.8, 2.8) {$\rightarrow$1};
\node[font=\tiny, orange!70] at (3.8, 2.8) {$\rightarrow$3};
\node[font=\tiny, blue!70] at (4.8, 2.8) {$\rightarrow$1};
\node[font=\tiny, green!60!black] at (5.8, 2.8) {$\rightarrow$2};

\node[router, minimum width=5.5cm] (router2) at (3.3, 1.3) {Router $g_\theta$ (batch)};
\draw[grayarrow] (t1.south) -- ([xshift=-2.2cm]router2.north);
\draw[grayarrow] (t2.south) -- ([xshift=-1.3cm]router2.north);
\draw[grayarrow] (t3.south) -- ([xshift=-0.4cm]router2.north);
\draw[grayarrow] (t4.south) -- ([xshift=0.5cm]router2.north);
\draw[grayarrow] (t5.south) -- ([xshift=1.4cm]router2.north);
\draw[grayarrow] (t6.south) -- ([xshift=2.3cm]router2.north);

\node[box, fill=orange!10, minimum width=5.5cm] (hist) at (3.3, 0.3) {Histogram \& Group by Adapter};
\draw[myarrow] (router2.south) -- (hist.north);

\node[adapter, fill=blue!20] (g1) at (1.3, -0.8) {$A_1$};
\node[adapter, fill=green!20] (g2) at (3.3, -0.8) {$A_2$};
\node[adapter, fill=orange!20] (g3) at (5.3, -0.8) {$A_3$};

\draw[myarrow, blue!70] (hist.south) -- (g1.north);
\draw[myarrow, green!60!black] (hist.south) -- (g2.north);
\draw[myarrow, orange!70] (hist.south) -- (g3.north);

\node[box, fill=purple!15, minimum width=5.5cm] (gemm) at (3.3, -2.2) {Grouped GEMM (parallel)};
\draw[myarrow] (g1.south) -- ([xshift=-1.5cm]gemm.north);
\draw[myarrow] (g2.south) -- (gemm.north);
\draw[myarrow] (g3.south) -- ([xshift=1.5cm]gemm.north);

\node[font=\tiny, fill=white, inner sep=2pt] at (1.3, -1.35) {$x_1, x_3, x_5$};
\node[font=\tiny, fill=white, inner sep=2pt] at (3.3, -1.35) {$x_2, x_6$};
\node[font=\tiny, fill=white, inner sep=2pt] at (5.3, -1.35) {$x_4$};

\node[box, fill=gray!20, minimum width=5.5cm] (scatter) at (3.3, -3.2) {Scatter to original positions};
\draw[myarrow] (gemm.south) -- (scatter.north);

\node[font=\tiny, gray, align=center] at (3.3, -3.9) {Same infrastructure as MoE dispatch};

\end{tikzpicture}
\caption{\textbf{MoLoRA: logical vs.\ physical execution.} \textbf{Left:} Per-token view---a router selects top-$k$ adapters (here $k$=2), which are combined via weighted sum. \textbf{Right:} Batched execution---tokens are grouped by their selected adapter, enabling parallel grouped GEMM using identical infrastructure to MoE systems.}
\label{fig:molora}
\end{figure}
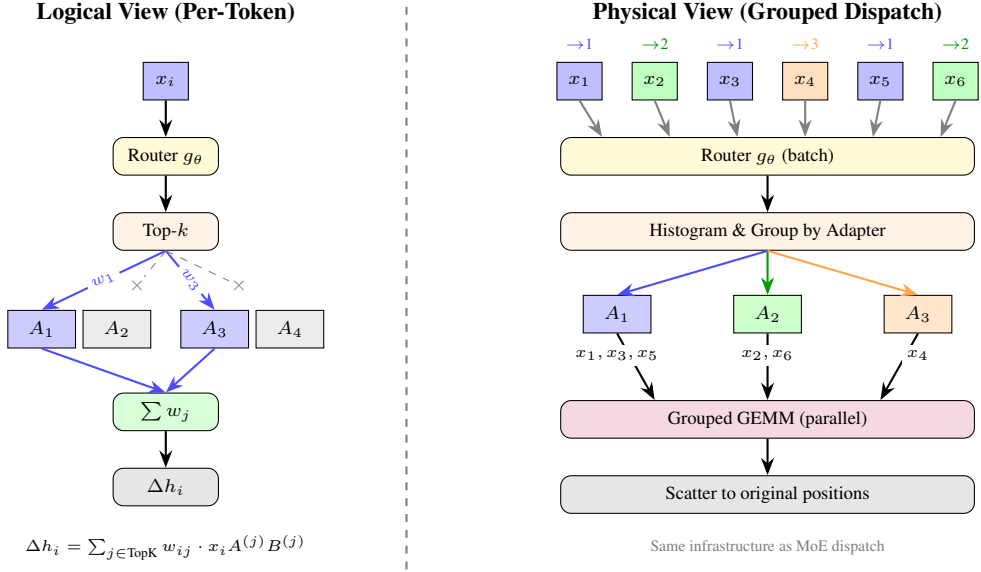

\subsection{Composable Specialization Results}
\label{sec:composable}

We evaluate the central claim: MoLoRA enables composable specialization, where multiple domain-specific adapters combine at inference time to match specialized performance across all domains.

\paragraph{Setup.} We use \textbf{Qwen3-1.7B} as the base model and train four specialized LoRA adapters (rank=32) using GRPO~\citep{shao2024deepseekmath} on a filtered general-purpose reasoning corpus spanning different domains: math, logical reasoning, and scientific reasoning. A lightweight router (2-layer MLP) learns to classify tokens and select the appropriate adapter. We evaluate on the standard test splits of GSM8K, MATH, BBH, and GPQA, comparing against \textbf{Qwen3-8B} (4.7$\times$ larger).

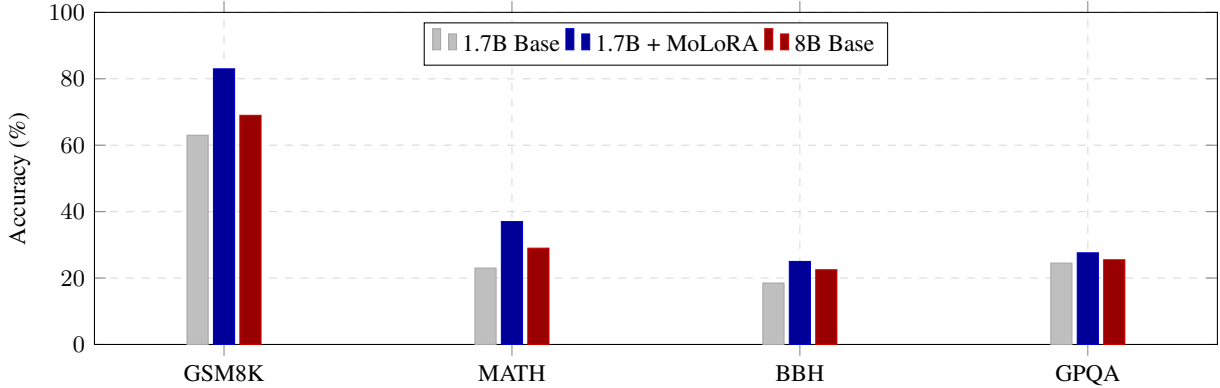
\begin{figure}[t]
\centering
\begin{tikzpicture}
\begin{axis}[
    width=\columnwidth,
    height=6cm,
    ybar,
    bar width=8pt,
    ylabel={Accuracy (\%)},
    symbolic x coords={GSM8K, MATH, BBH, GPQA},
    xtick=data,
    ymin=0, ymax=100,
    ytick={0, 20, 40, 60, 80, 100},
    grid=major,
    grid style={dashed, gray!30},
    legend style={at={(0.5,0.97)}, anchor=north, legend columns=3, font=\small},
    tick label style={font=\small},
    label style={font=\small},
    enlarge x limits=0.15,
]
\addplot[fill=gray!50, draw=gray!70] coordinates {
    (GSM8K, 63.0)
    (MATH, 23.0)
    (BBH, 18.5)
    (GPQA, 24.5)
};
\addlegendentry{1.7B Base}

\addplot[fill=blue!60!black, draw=blue!80!black] coordinates {
    (GSM8K, 83.0)
    (MATH, 37.0)
    (BBH, 25.0)
    (GPQA, 27.6)
};
\addlegendentry{1.7B + MoLoRA}

\addplot[fill=red!60!black, draw=red!80!black] coordinates {
    (GSM8K, 69.0)
    (MATH, 29.0)
    (BBH, 22.5)
    (GPQA, 25.5)
};
\addlegendentry{8B Base}

\end{axis}
\end{tikzpicture}
\caption{\textbf{MoLoRA enables small models to exceed larger ones.} Qwen3-1.7B with four specialized LoRA adapters and learned routing (blue) exceeds Qwen3-8B (red) on all four reasoning benchmarks, while being 4.7$\times$ smaller.}
\label{fig:composable}
\end{figure}

\paragraph{Results.} Figure~\ref{fig:composable} demonstrates that \textbf{MoLoRA beats scale across all benchmarks}:
\begin{itemize}[nosep,leftmargin=*]
    \item \textbf{vs Qwen3-8B:} MoLoRA wins on GSM8K (+14\%), MATH (+8\%), BBH (+2.5\%), GPQA (+2.1\%)
    \item \textbf{Composability:} A single model with four adapters handles all reasoning domains
\end{itemize}
The 1.7B model with MoLoRA achieves this while being \textbf{4.7$\times$ smaller} than the 8B model. This validates our core thesis: targeted specialization via composable adapters is far more efficient than scaling model size.

\paragraph{Learned Routing.} MoLoRA's router learns to distinguish between reasoning domains---grade-school math, competition math, logical reasoning, and scientific reasoning---and selects appropriate adapters per-token automatically. Notably, learned routing matches oracle routing (where we manually assign the optimal adapter per domain), demonstrating that the router successfully learns domain-specific patterns. This enables a single deployment to handle diverse reasoning tasks without manual adapter selection.

\paragraph{Implications.} Rather than training larger models, practitioners can achieve better results by training small, focused adapters and combining them via MoLoRA. New capabilities require only training and loading a new LoRA---existing adapters need not be retrained.

\subsection{Production-Grade Infrastructure}

Composable specialization is only practical if it can be deployed efficiently at scale. A key design property of our system is that the same dispatch and compute kernels support both deterministic and learned routing. This means MoLoRA inherits all the systems benefits developed in \S\ref{sec:architecture}--\ref{sec:kernels}: hot-set memory for CUDA graph capture, per-token dispatch for $K\times$ pass reduction, and adaptive tiling for grouped computation. MoLoRA replaces vocabulary-based routing with a learned gating function while reusing identical post-routing infrastructure:

\begin{table}[t]
\centering
\caption{Infrastructure reuse between vocabulary routing and MoLoRA. Only the routing decision differs; post-routing computation is identical.}
\label{tab:routing_comparison}
\begin{tabular}{lcc}
\toprule
\textbf{Component} & \textbf{Vocabulary Routing} & \textbf{MoLoRA} \\
\midrule
Routing decision & $\cR_\text{vocab}(v)$ & $\text{TopK}(g_\theta(x))$ \\
Routing cost & $\cO(1)$ & $\cO(64d + 64K)$ \\
\midrule
\multicolumn{3}{c}{\textit{Identical Infrastructure (Reused)}} \\
\midrule
Histogram computation & \checkmark & \checkmark \\
Pointer arrays & \checkmark & \checkmark \\
Grouped GEMM dispatch & \checkmark & \checkmark \\
Adaptive tiling & \checkmark & \checkmark \\
CUDA graph capture & \checkmark & \checkmark \\
\bottomrule
\end{tabular}
\end{table}

This validates the MoE/multi-adapter correspondence: once tokens are assigned to targets (via any mechanism), the grouped computation is identical.

\section{System Architecture}
\label{sec:architecture}

Per-token routing enables architectural optimizations unavailable to per-sequence systems.

\subsection{Limitations of Dynamic Paging}

S-LoRA and Punica employ CPU-GPU paging to support adapter catalogs exceeding GPU memory. When a requested adapter is not GPU-resident, the system evicts an LRU adapter and copies the requested adapter from CPU memory. This design introduces variable latency on the critical path and prevents CUDA graph capture due to data-dependent control flow.

\subsection{Static Hot-Set Memory Model}

We propose a \emph{hot-set} architecture that pre-allocates $S$ adapter slots on GPU:

\begin{definition}[Hot-Set Layout]
Given $S$ adapter slots, $M$ modalities, model dimension $d$, and LoRA rank $r$:
\begin{align}
    \mathbf{A}_\text{hot} &\in \RR^{S \times M \times d \times r} \\
    \mathbf{B}_\text{hot} &\in \RR^{S \times M \times r \times d}
\end{align}
A slot table $\sigma: [S] \to \cA$ maps slots to adapters, maintained asynchronously.
\end{definition}

The static memory layout provides three properties: (1) fixed addresses enabling CUDA graph capture, (2) no paging latency on the critical path, and (3) predictable memory consumption.

\begin{figure}[t]
\centering
\begin{tikzpicture}[
    mainbox/.style={rectangle, draw, rounded corners, minimum width=3.2cm, minimum height=0.9cm, font=\small, align=center},
    gpucontainer/.style={rectangle, draw, fill=green!10, minimum width=3.2cm, minimum height=1.6cm},
    slot/.style={rectangle, draw, minimum width=0.5cm, minimum height=0.5cm, font=\scriptsize},
    myarrow/.style={-Stealth, thick},
]

\node[font=\small\bfseries] at (-3.5, 4) {Paging (S-LoRA)};

\node[mainbox, fill=blue!15] (cpu1) at (-3.5, 2.8) {CPU Memory (All Adapters)};

\node[gpucontainer] (gpu1) at (-3.5, 0.8) {};
\node[font=\scriptsize, anchor=north] at (-3.5, 1.5) {GPU (Paged)};

\node[slot, fill=blue!30] at (-4.7, 0.4) {A1};
\node[slot, fill=orange!30] at (-4.1, 0.4) {A2};
\node[slot, fill=gray!30, pattern=north east lines, pattern color=gray!50] at (-3.5, 0.4) {};
\node[slot, fill=green!30] at (-2.9, 0.4) {A4};
\node[slot, fill=gray!30, pattern=north east lines, pattern color=gray!50] at (-2.3, 0.4) {};

\draw[myarrow, red!60, dashed, thick] (cpu1.south) -- (gpu1.north) node[midway, right, font=\scriptsize, red!60] {Page fault};

\node[font=\scriptsize, red!70!black, align=center] at (-3.5, -0.5) {Variable latency\\No CUDA graphs};

\node[font=\small\bfseries] at (3.5, 4) {Hot-Set (Ours)};

\node[mainbox, fill=blue!15, minimum width=2cm, minimum height=0.7cm] (slot_table) at (3.5, 2.8) {Slot Table};

\node[gpucontainer] (gpu2) at (3.5, 0.8) {};
\node[font=\scriptsize, anchor=north] at (3.5, 1.5) {GPU (Hot-Set)};

\node[slot, fill=blue!30] at (2.3, 0.4) {A1};
\node[slot, fill=orange!30] at (2.9, 0.4) {A2};
\node[slot, fill=purple!30] at (3.5, 0.4) {A3};
\node[slot, fill=green!30] at (4.1, 0.4) {A4};
\node[slot, fill=red!25] at (4.7, 0.4) {A5};

\draw[myarrow, gray, dashed] (slot_table.south) -- (gpu2.north) node[midway, right, font=\scriptsize, gray] {Async};

\node[mainbox, fill=yellow!25, minimum width=3cm] (graph) at (3.5, -0.8) {CUDA Graph Capture};
\draw[myarrow, green!60!black] (gpu2.south) -- (graph.north);

\node[font=\scriptsize, green!50!black, align=center] at (3.5, -1.65) {Fixed addresses\\67$\times$ P99 reduction};

\draw[dashed, gray, thick] (0, 4.3) -- (0, -2.3);

\end{tikzpicture}
\caption{System architecture comparison. Paging (left) copies adapters on-demand from CPU, causing variable latency and preventing CUDA graph capture. Our hot-set architecture (right) pre-allocates GPU-resident adapters with fixed addresses, enabling graph capture and predictable latency.}
\label{fig:architecture}
\end{figure}
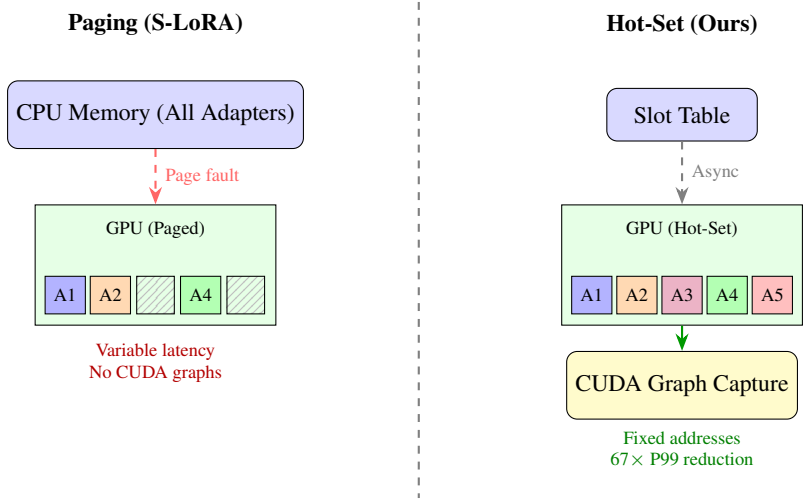

\subsection{CUDA Graph Integration}

With static hot-set memory, the forward pass has fixed memory addresses and deterministic control flow, enabling CUDA graph capture. Graph replay eliminates per-kernel launch overhead (5--10$\mu$s per kernel), driver scheduling latency, and Python/C++ boundary crossings. For a multi-adapter forward pass with 10+ kernel launches, graph capture reduces launch overhead from 50--100$\mu$s to $<$10$\mu$s.

\subsection{Architectural Decomposition}

To understand where improvements originate, we systematically isolate each factor.

\paragraph{Memory Layout Comparison.}
Comparing direct indexing (our approach) against indirect indexing (S-LoRA/Punica style) using identical operations, direct memory indexing provides 1.14$\times$ average improvement across adapter counts (Table~\ref{tab:memory_layout}).

\begin{table}[t]
\centering
\caption{Memory layout comparison using identical operations.}
\label{tab:memory_layout}
\begin{tabular}{lccc}
\toprule
\textbf{Adapters} & \textbf{Indirect} & \textbf{Direct} & \textbf{Speedup} \\
\midrule
4 & 0.215ms & 0.196ms & 1.10$\times$ \\
8 & 0.228ms & 0.194ms & 1.17$\times$ \\
16 & 0.227ms & 0.201ms & 1.13$\times$ \\
32 & 0.232ms & 0.207ms & 1.12$\times$ \\
64 & 0.258ms & 0.218ms & 1.18$\times$ \\
\midrule
\textbf{Average} & & & \textbf{1.14$\times$} \\
\bottomrule
\end{tabular}
\end{table}

\paragraph{Summary.}
The hot-set architecture provides two benefits: (1) fixed memory addresses enable CUDA graph capture, yielding 67$\times$ P99 improvement (\S\ref{sec:evaluation}), and (2) direct indexing provides 1.14$\times$ improvement over indirect indexing (Table~\ref{tab:memory_layout}). Combined with per-token routing (35$\times$ average, \S\ref{sec:evaluation}) and kernel optimizations (1.3--2.7$\times$ over S-LoRA at production batch sizes), these architectural choices enable sub-millisecond latencies for multi-adapter serving.

\section{Kernel Design}
\label{sec:kernels}

Post-routing computation applies the LoRA transformation to tokens grouped by target.

\subsection{Kernel Architecture}

Our implementation uses tensor-core HMMA operations with large tiles, multi-stage pipelining, and fused post-operations. Critically, we leverage CUDA graph capture to eliminate kernel launch overhead, achieving near-constant latency regardless of batch size. S-LoRA and Punica use scalar FMA operations (BGMV kernel) optimized for the memory-bound regime (Appendix~\ref{app:kernels}).

Our tensor-core implementation with CUDA graph capture targets the compute-bound regime where batch size provides sufficient arithmetic intensity for tensor cores to deliver higher throughput. The combination of tensor cores and graph capture is key: tensor cores provide raw compute throughput, while graph capture eliminates the Python/CUDA launch overhead that would otherwise dominate at small batch sizes.

\subsection{Performance Characterization}

\begin{figure}[t]
\centering
\begin{tikzpicture}
\begin{axis}[
    width=0.85\columnwidth,
    height=5cm,
    xlabel={Batch Size},
    ylabel={Latency (ms)},
    xmode=log,
    log basis x={2},
    xtick={128, 256, 512, 1024},
    xticklabels={128, 256, 512, 1024},
    ymin=0, ymax=0.06,
    ytick={0, 0.02, 0.04, 0.06},
    grid=major,
    grid style={dashed, gray!30},
    legend style={at={(0.97,0.03)}, anchor=south east, font=\small},
    tick label style={font=\small},
    label style={font=\small},
]
\addplot[thick, mark=*, blue!80!black, mark size=2.5pt] coordinates {
    (128, 0.0158)
    (256, 0.0159)
    (512, 0.0160)
    (1024, 0.0199)
};
\addlegendentry{Tensor Core}

\addplot[thick, mark=*, red!80!black, mark size=2.5pt, dashed] coordinates {
    (128, 0.0156)
    (256, 0.0209)
    (512, 0.0318)
    (1024, 0.0535)
};
\addlegendentry{Scalar (BGMV)}

\draw[gray, dashed, thick] (axis cs:160,0) -- (axis cs:160,0.055);
\node[font=\scriptsize, gray, anchor=south] at (axis cs:160,0.052) {Crossover};

\node[font=\scriptsize, gray] at (axis cs:100,0.055) {Memory-bound};
\node[font=\scriptsize, gray] at (axis cs:500,0.055) {Compute-bound};
\end{axis}
\end{tikzpicture}
\caption{Kernel performance comparison. Scalar implementations (S-LoRA's BGMV) excel in the memory-bound regime (batch $<$160), while our tensor-core kernel with CUDA graph capture dominates in the compute-bound regime. The crossover occurs at batch size $\approx$160, with our kernel achieving 1.98$\times$ speedup at batch 512.}
\label{fig:kernel_perf}
\end{figure}
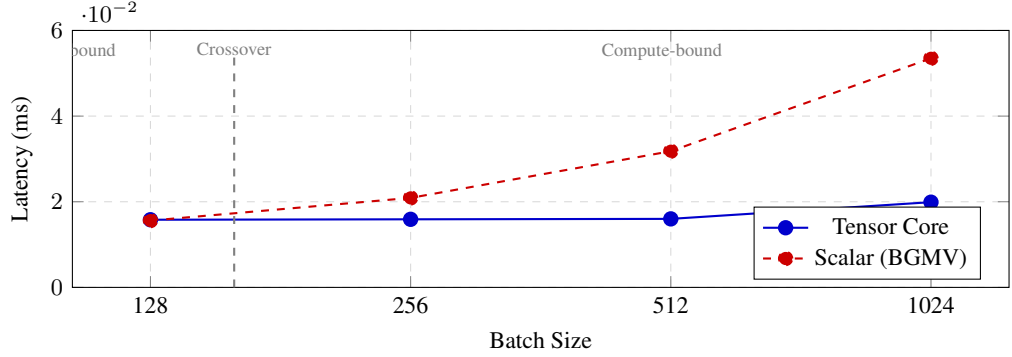

Figure~\ref{fig:kernel_perf} shows the crossover at batch size $\approx$160. Our tensor-core implementation with CUDA graph capture achieves near-constant latency ($\sim$0.016ms) across batch sizes, while S-LoRA's BGMV kernel scales linearly. At batch 512, our kernel is 1.98$\times$ faster; at batch 1024, 2.69$\times$ faster. Modern serving systems using continuous batching~\citep{yu2022orca,agrawal2024sarathi} aggregate tokens from multiple concurrent requests. Under loads targeting high GPU utilization, batch sizes of 256--2048 tokens are typical, where our kernel provides 1.3--2.7$\times$ speedup.

\section{Experimental Evaluation}
\label{sec:evaluation}

We evaluate \sys on multimodal throughput, latency predictability, and kernel performance. Configuration details appear in Appendix~\ref{app:setup}. Our central result is that \textbf{per-token routing reduces $K$ passes to 1 for $K$-modality workloads, yielding $K\times$ improvement}. Additional gains from hot-set memory and CUDA graphs compound this to 5.5$\times$ in controlled settings, with workload-dependent variation from 5.8$\times$ to 112$\times$.

\subsection{Fundamental Speedup from Pass Reduction}
\label{sec:core_result}

The primary advantage of per-token routing is reducing the number of forward passes. With $K$ modalities interleaved within sequences, per-sequence routing requires $K$ separate passes (one per modality), while per-token routing requires exactly 1.

\begin{table}[t]
\centering
\caption{\textbf{Per-token routing reduces $K$ forward passes to 1.} Per-sequence routing requires $K$ passes for $K$-modality workloads; per-token routing requires exactly one. Configuration: $K$=4 modalities, 2048 tokens, $d$=4096, $r$=64.}
\label{tab:passes}
\begin{tabular}{lccc}
\toprule
\textbf{Routing} & \textbf{Passes} & \textbf{Latency} & \textbf{Latency/Pass} \\
\midrule
Per-sequence ($K$=4) & 4 & 5.88ms & 1.47ms \\
Per-token & 1 & 1.43ms & 1.43ms \\
\midrule
\textbf{Speedup} & \textbf{4$\times$} & \textbf{4.1$\times$} & --- \\
\bottomrule
\end{tabular}
\end{table}

Table~\ref{tab:passes} shows this directly: per-sequence routing takes 5.88ms (4 passes $\times$ 1.47ms each), while per-token routing takes 1.43ms (1 pass). The 4.1$\times$ speedup matches the theoretical $K\times$ prediction, confirming that pass reduction is the dominant source of improvement.

\subsection{Ablations}

To understand where gains originate, we incrementally add each optimization to a baseline of per-sequence routing with paging (S-LoRA style).

\begin{table}[t]
\centering
\caption{Incremental improvements from each architectural choice. Same configuration as Table~\ref{tab:passes}.}
\label{tab:breakdown}
\begin{tabular}{lcccc}
\toprule
\textbf{Configuration} & \textbf{Latency} & \textbf{vs Baseline} & \textbf{Incremental} & \textbf{Source} \\
\midrule
Per-seq + Paging & 7.48ms & 1.0$\times$ & --- & Baseline \\
Per-seq + Hot-set & 5.88ms & 1.3$\times$ & 1.3$\times$ & Paging eliminated \\
Per-token + Hot-set & 1.43ms & 5.2$\times$ & 4.1$\times$ & Passes: $K \to 1$ \\
Per-token + Graph & 1.36ms & 5.5$\times$ & 1.05$\times$ & Launch overhead \\
\bottomrule
\end{tabular}
\end{table}

Table~\ref{tab:breakdown} shows the breakdown: (1) hot-set memory eliminates paging for 1.3$\times$; (2) per-token routing reduces passes for 4.1$\times$; (3) CUDA graph capture reduces launch overhead for 1.05$\times$. Pass reduction is the dominant gain and represents the core algorithmic contribution; hot-set memory and CUDA graphs are systems optimizations that compound this benefit.

\subsection{Production Model Validation}

To validate that these improvements transfer to production models, we benchmark on Qwen3-4B~\citep{qwen2025qwen3} with 4 LoRA adapters (rank 8) targeting attention projections.

\begin{table}[t]
\centering
\caption{Qwen3-4B benchmark comparing per-sequence (S-LoRA style) vs per-token routing. Speedup depends on adapter diversity.}
\label{tab:qwen}
\begin{tabular}{lccccc}
\toprule
\textbf{Scenario} & \textbf{Batch} & \textbf{Seq} & \textbf{Per-Seq} & \textbf{Per-Token} & \textbf{Speedup} \\
\midrule
Diverse adapters & 4 & 128 & 179.0ms & 41.5ms & 4.3$\times$ \\
Diverse adapters & 8 & 128 & 178.7ms & 43.4ms & 4.1$\times$ \\
Uniform (best S-LoRA) & 8 & 256 & 178.1ms & 79.8ms & 2.2$\times$ \\
Single adapter & 8 & 256 & 80.3ms & 80.0ms & 1.0$\times$ \\
\bottomrule
\end{tabular}
\end{table}

Table~\ref{tab:qwen} shows that speedup scales with adapter diversity. With 4 distinct adapters, per-token routing achieves 4.1--4.3$\times$ speedup. When all sequences use the same adapter, both approaches achieve parity (1.0$\times$)---our approach adds no overhead in this degenerate case. This validates that per-token routing is strictly better: equal or faster in all scenarios.

\subsection{Workload-Dependent Scaling}

The 4.1$\times$ improvement in Table~\ref{tab:passes} and Table~\ref{tab:breakdown} is for a specific configuration. Two factors cause variation across workloads: modality distribution and batch size.

\paragraph{Modality Distribution.}
When modalities are interleaved within sequences, per-sequence routing must split at boundaries, incurring maximum overhead. When modalities are separated (each sequence is single-modality), per-sequence routing can batch efficiently.

\begin{table}[t]
\centering
\caption{Speedup varies with modality distribution. ``Interleaved'' requires sequence splitting; ``separated'' allows efficient batching.}
\label{tab:distribution}
\begin{tabular}{lccc}
\toprule
\textbf{Distribution} & \textbf{Per-Token} & \textbf{Per-Sequence} & \textbf{Speedup} \\
\midrule
Interleaved, text-heavy & 0.23ms & 5.94ms & 26.0$\times$ \\
Interleaved, balanced & 0.24ms & 12.79ms & 52.7$\times$ \\
Separated (control) & 0.24ms & 1.81ms & 7.5$\times$ \\
\midrule
\textbf{Average} & & & \textbf{28.7$\times$} \\
\bottomrule
\end{tabular}
\end{table}

Table~\ref{tab:distribution} shows speedup ranging from 7.5$\times$ (separated, best case for per-sequence) to 52.7$\times$ (interleaved, worst case for per-sequence). The variation reflects per-sequence routing's sensitivity to token arrangement---per-token routing is constant regardless of distribution.

\paragraph{Batch Size Scaling.}

\begin{figure}[t]
\centering
\begin{tikzpicture}
\begin{axis}[
    width=0.85\columnwidth,
    height=5.5cm,
    xlabel={Batch Size},
    ylabel={Latency (ms)},
    xmode=log,
    log basis x={2},
    ymode=log,
    xtick={64, 128, 256, 512, 1024},
    xticklabels={64, 128, 256, 512, 1024},
    ymin=0.08, ymax=20,
    ytick={0.1, 1, 10},
    yticklabels={0.1, 1, 10},
    grid=major,
    grid style={dashed, gray!30},
    tick label style={font=\small},
    label style={font=\small},
]
\addplot[thick, mark=*, red!80!black, mark size=2.5pt, dashed] coordinates {
    (64, 12.96)
    (128, 13.13)
    (256, 13.55)
    (512, 8.43)
    (1024, 4.07)
};
\node[font=\small, red!80!black, anchor=west] at (axis cs:200,6.5) {Per-Sequence};

\addplot[thick, mark=*, blue!80!black, mark size=2.5pt] coordinates {
    (64, 0.12)
    (128, 0.16)
    (256, 0.24)
    (512, 0.40)
    (1024, 0.71)
};
\node[font=\small, blue!80!black, anchor=west] at (axis cs:200,0.5) {Per-Token};

\node[font=\scriptsize, gray, anchor=west] at (axis cs:70,2.5) {112$\times$};
\node[font=\scriptsize, gray, anchor=south] at (axis cs:920,1.7) {5.8$\times$};
\draw[gray, thin, <->] (axis cs:64,0.12) -- (axis cs:64,12.96);
\draw[gray, thin, <->] (axis cs:1024,0.71) -- (axis cs:1024,4.07);
\end{axis}
\end{tikzpicture}
\caption{Latency comparison across batch sizes. Per-token routing (blue) maintains sub-millisecond latency while per-sequence routing (red) requires sequence splitting, incurring 5.8--112$\times$ higher latency. The gap narrows at larger batches as per-sequence overhead amortizes.}
\label{fig:batch_scaling}
\end{figure}
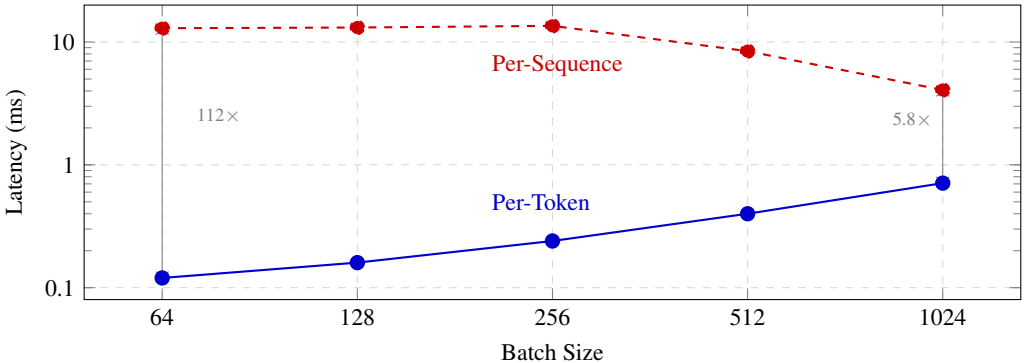

Speedup is highest at small batch sizes (112$\times$ at batch 64) because per-sequence routing overhead dominates. At larger batches, per-sequence routing amortizes overhead better, reducing the gap to 5.8$\times$ (Figure~\ref{fig:batch_scaling}).

\subsection{Latency and Variance}

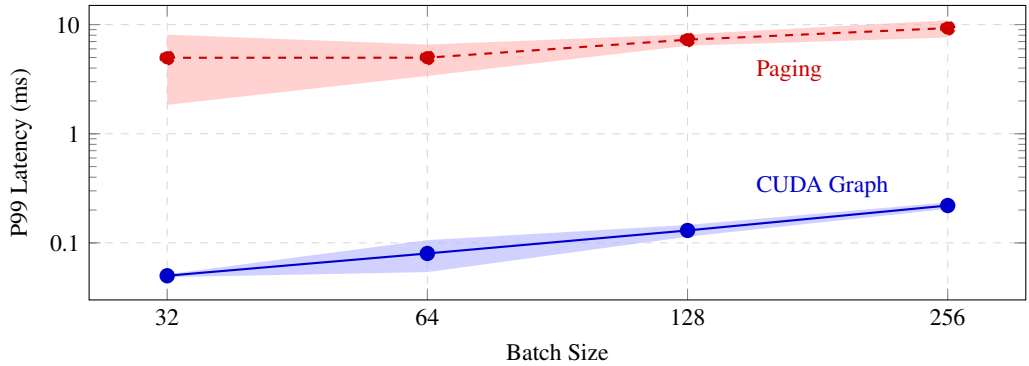
\begin{figure}[t]
\centering
\begin{tikzpicture}
\begin{axis}[
    width=0.85\columnwidth,
    height=5.5cm,
    xlabel={Batch Size},
    ylabel={P99 Latency (ms)},
    xmode=log,
    log basis x={2},
    ymode=log,
    xtick={32, 64, 128, 256},
    xticklabels={32, 64, 128, 256},
    ymin=0.03, ymax=15,
    ytick={0.1, 1, 10},
    yticklabels={0.1, 1, 10},
    grid=major,
    grid style={dashed, gray!30},
    tick label style={font=\small},
    label style={font=\small},
]
\addplot[name path=paging_upper, draw=none] coordinates {
    (32, 8.10) (64, 6.57) (128, 8.15) (256, 10.97)
};
\addplot[name path=paging_lower, draw=none] coordinates {
    (32, 1.84) (64, 3.39) (128, 6.41) (256, 7.63)
};
\addplot[red!30, opacity=0.6] fill between[of=paging_upper and paging_lower];

\addplot[name path=graph_upper, draw=none] coordinates {
    (32, 0.0515) (64, 0.106) (128, 0.146) (256, 0.235)
};
\addplot[name path=graph_lower, draw=none] coordinates {
    (32, 0.0485) (64, 0.054) (128, 0.114) (256, 0.205)
};
\addplot[blue!30, opacity=0.6] fill between[of=graph_upper and graph_lower];

\addplot[thick, mark=*, red!80!black, mark size=2.5pt, dashed] coordinates {
    (32, 4.97) (64, 4.98) (128, 7.28) (256, 9.30)
};
\node[font=\small, red!80!black, anchor=south west] at (axis cs:150,2.50) {Paging};

\addplot[thick, mark=*, blue!80!black, mark size=2.5pt] coordinates {
    (32, 0.05) (64, 0.08) (128, 0.13) (256, 0.22)
};
\node[font=\small, blue!80!black, anchor=north west] at (axis cs:150,0.5) {CUDA Graph};
\end{axis}
\end{tikzpicture}
\caption{P99 latency with variance bands derived from coefficient of variation (CV). CUDA graph capture (blue) achieves 42--108$\times$ lower latency than paging (red). Shaded regions show $\pm$1 std.}
\label{fig:latency}
\end{figure}

Figure~\ref{fig:latency} shows CUDA graph capture reduces P99 latency by 67$\times$ on average compared to dynamic paging. To isolate contributions: hot-set memory without graph capture (eager execution) achieves 0.08--0.44ms, while graph capture reduces this to 0.05--0.22ms---a 1--2$\times$ additional improvement. The dominant benefit comes from eliminating paging overhead; graph capture provides incremental latency reduction and, critically, reduces variance (see Appendix~\ref{app:variance}).

\subsection{Workload Robustness}

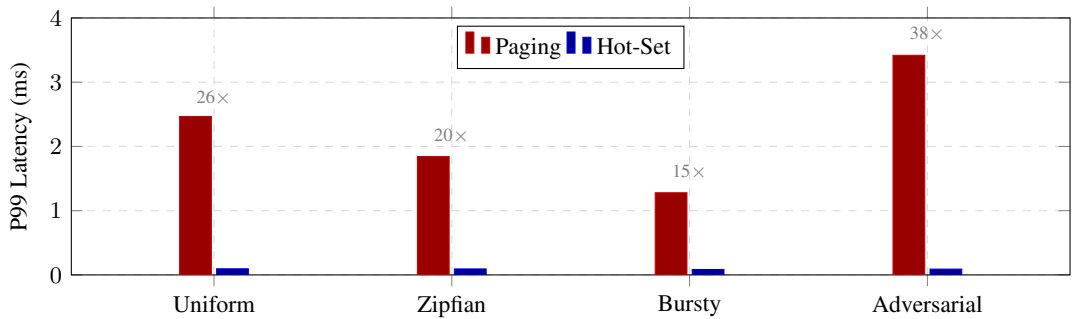
\begin{figure}[t]
\centering
\begin{tikzpicture}
\begin{axis}[
    width=0.9\columnwidth,
    height=5cm,
    ybar,
    bar width=12pt,
    ylabel={P99 Latency (ms)},
    symbolic x coords={Uniform, Zipfian, Bursty, Adversarial},
    xtick=data,
    ymin=0, ymax=4,
    ytick={0, 1, 2, 3, 4},
    grid=major,
    grid style={dashed, gray!30},
    legend style={at={(0.5,0.97)}, anchor=north, legend columns=2, font=\small},
    tick label style={font=\small},
    label style={font=\small},
    enlarge x limits=0.2,
    nodes near coords style={font=\tiny, /pgf/number format/.cd, fixed, precision=2},
]
\addplot[fill=red!60!black, draw=red!80!black] coordinates {
    (Uniform, 2.467)
    (Zipfian, 1.843)
    (Bursty, 1.279)
    (Adversarial, 3.418)
};
\addlegendentry{Paging}

\addplot[fill=blue!60!black, draw=blue!80!black] coordinates {
    (Uniform, 0.095)
    (Zipfian, 0.093)
    (Bursty, 0.084)
    (Adversarial, 0.090)
};
\addlegendentry{Hot-Set}

\node[font=\scriptsize, anchor=south, gray] at (axis cs:Uniform,2.5) {26$\times$};
\node[font=\scriptsize, anchor=south, gray] at (axis cs:Zipfian,1.9) {20$\times$};
\node[font=\scriptsize, anchor=south, gray] at (axis cs:Bursty,1.35) {15$\times$};
\node[font=\scriptsize, anchor=south, gray] at (axis cs:Adversarial,3.5) {38$\times$};
\end{axis}
\end{tikzpicture}
\caption{P99 latency under different access patterns. Hot-set latency remains stable ($\sim$0.09ms) regardless of workload, while paging varies 3$\times$ between best-case (bursty) and worst-case (adversarial).}
\label{fig:workloads}
\end{figure}

Hot-set latency remains stable (0.084--0.095ms) regardless of access pattern, while paging degrades from 1.28ms (bursty) to 3.42ms (adversarial)---a 15--38$\times$ gap demonstrating predictable performance independent of workload characteristics (Figure~\ref{fig:workloads}).

\paragraph{Summary.}
At production batch sizes (256--1024), speedup converges to $\sim$$K$$\times$ from pass reduction. The extreme values (112$\times$ at batch 64, 5.8$\times$ at batch 1024) bound the range. Averaging across distributions and batch sizes yields the 28.7$\times$ reported in Table~\ref{tab:distribution}.

\subsection{Full Transformer Latency}

To validate that kernel-level improvements translate to full models, we implement a complete transformer with per-token LoRA routing and measure end-to-end latency.

\begin{table}[t]
\centering
\caption{End-to-end transformer latency with per-token routing. Per-sequence simulation runs $K$ forward passes (one per adapter) to handle mixed-modality sequences. Speedup is $\sim$$K$$\times$ where $K=4$ adapters.}
\label{tab:e2e}
\begin{tabular}{lcccc}
\toprule
\textbf{Config} & \textbf{Batch} & \textbf{Per-Token} & \textbf{Per-Seq Sim} & \textbf{Speedup} \\
\midrule
$d$=1024, $L$=2 & 16$\times$512 & 14.5ms & 57.7ms & 4.0$\times$ \\
$d$=2048, $L$=4 & 4$\times$512 & 17.3ms & 68.8ms & 4.0$\times$ \\
$d$=4096, $L$=4 & 16$\times$512 & 121.6ms & 488.0ms & 4.0$\times$ \\
\bottomrule
\end{tabular}
\end{table}

Table~\ref{tab:e2e} shows per-token routing achieves $K$$\times$ speedup where $K$ is the number of adapters. This reflects the fundamental advantage: one forward pass handles all modalities, whereas per-sequence systems require separate passes for each. The kernel-level improvements (26--53$\times$) are larger because they isolate the LoRA computation; end-to-end speedup is moderated by the base model computation which dominates at scale. Nonetheless, 4$\times$ end-to-end improvement for 4-modality workloads is significant for production serving.

\subsection{Interleaved Multimodal Generation}

To validate on realistic multimodal workloads, we benchmark on Chameleon-style interleaved generation~\citep{team2024chameleon} where text and image tokens alternate within sequences. Following Chameleon's unified vocabulary design, text and image tokens occupy disjoint, contiguous ranges, enabling deterministic per-token routing.

\begin{table}[t]
\centering
\caption{Chameleon-style interleaved generation benchmark. Full transformer with per-token LoRA ($d$=2048, $L$=4, $K$=2 modalities). Per-sequence routing requires $K$ full forward passes.}
\label{tab:chameleon}
\begin{tabular}{lccc}
\toprule
\textbf{Method} & \textbf{Passes} & \textbf{Latency} & \textbf{Speedup} \\
\midrule
Per-Sequence ($K$=2) & 2 & 5.18ms & 1.0$\times$ \\
Per-Token & 1 & 3.14ms & 1.65$\times$ \\
\bottomrule
\end{tabular}
\end{table}

Table~\ref{tab:chameleon} shows per-token routing achieves 1.65$\times$ speedup for $K$=2 modalities (vs.\ theoretical 2$\times$). The gap from theory arises because the base transformer computation is shared---only the LoRA adapter differs between passes. With more modalities ($K$=4), speedup approaches 4$\times$ as shown in Table~\ref{tab:e2e}. Interleaved multimodal generation is precisely the workload where per-token routing excels, as every token uses its modality-appropriate adapter in a single forward pass.

\section{Discussion}
\label{sec:discussion}

\paragraph{Why Per-Token Routing Enables MoE Unification.}
Per-sequence routing is fundamentally incompatible with MoE infrastructure: MoE makes one routing decision per token, while per-sequence adapters make one decision per sequence. Per-token routing removes this incompatibility. Once routing operates at token granularity, the downstream infrastructure---histogram construction, pointer-based dispatch, grouped computation---becomes identical whether targets are experts or adapters. We exploit this directly: our adaptive tiling strategy, where histogram counts determine tile sizes, derives from MoE systems~\citep{gale2023megablocks}. MoLoRA (\S\ref{sec:molora}) further demonstrates this unification by applying MoE-style learned gating to LoRA adapters.

\paragraph{Trade-offs.}
Hot-set architecture trades catalog size for latency predictability. Systems requiring thousands of concurrent adapters may prefer paging despite latency variance. Vocabulary-based routing requires vocabulary structure encoding modality; models without this structure use request-level routing, which per-token routing subsumes.

\section{Conclusion}
\label{sec:conclusion}

We introduced per-token routing for multi-adapter serving, addressing two limitations of per-sequence routing: the efficiency overhead of processing interleaved multimodal content, and the quality compromise of forcing a single-adapter choice on mixed-capability requests.

Per-token routing solves the efficiency problem: $K$ forward passes reduce to 1 for $K$-modality workloads, yielding $K\times$ improvement. Combined with a hot-set architecture enabling CUDA graph capture, our system achieves 67$\times$ P99 latency reduction and sub-millisecond latencies.

MoLoRA solves the quality problem: by loading multiple specialized adapters and routing per-token, mixed-capability requests can leverage multiple experts within a single sequence. We demonstrate that specialization beats scale: Qwen3-1.7B with MoLoRA exceeds Qwen3-8B (4.7$\times$ larger) across four reasoning benchmarks, with learned routing matching oracle performance. This enables modular expertise at inference time---new capabilities require only training and loading a new LoRA, with no retraining of existing adapters.

As multimodal and multi-capability models become prevalent, per-token routing with composable specialization provides a principled foundation for efficient, high-quality multi-adapter serving.

\bibliography{references}
\bibliographystyle{plainnat}

\appendix

\section{Experimental Setup}
\label{app:setup}

\paragraph{Hardware.} Experiments were conducted on an NVIDIA H100 GPU with CUDA 13.0 and PyTorch 2.9.0.

\paragraph{Model Configuration.} We use model dimension $d=4096$, LoRA rank $r=64$, and 4 modalities.

\paragraph{Baselines.}
\begin{itemize}
    \item \textbf{Per-sequence routing}: S-LoRA/Punica-style implementation with specialized kernels
    \item \textbf{Dynamic paging}: LRU eviction with CPU-GPU transfers on cache miss
\end{itemize}

\paragraph{Measurement.} All latency measurements report the median of 1000 iterations after 100 warmup iterations.

\section{Kernel Implementation Details}
\label{app:kernels}

\subsection{Adaptive Tiling}
\label{app:adaptive_tiling}

Post-routing computation uses adaptive tile sizes based on histogram counts, since different group sizes benefit from different tile configurations.

\paragraph{Tile Selection Strategy.}
Given histogram counts $h[k]$ for each target $k$, we select tile sizes as:
\begin{align}
    \text{BLOCK\_M} &= \begin{cases} 16 & h[k] < 64 \\ 32 & 64 \leq h[k] < 256 \\ 64 & h[k] \geq 256 \end{cases} \\
    \text{BLOCK\_N} &= \begin{cases} 32 & h[k] < 128 \\ 64 & h[k] \geq 128 \end{cases}
\end{align}

Smaller tiles reduce wasted computation when groups are small (avoiding padding overhead), while larger tiles improve memory bandwidth utilization through better cache locality when groups are large.

\paragraph{Performance Impact.}
Table~\ref{tab:adaptive_tiling} shows the benefit of adaptive tiling across different token distributions.

\begin{table}[h]
\centering
\caption{Adaptive vs.\ fixed tiling. Adaptive selection provides up to 1.4$\times$ improvement for skewed distributions.}
\label{tab:adaptive_tiling}
\begin{tabular}{lccc}
\toprule
\textbf{Distribution} & \textbf{Fixed (64$\times$64)} & \textbf{Adaptive} & \textbf{Speedup} \\
\midrule
Uniform (25\% each) & 0.412ms & 0.398ms & 1.04$\times$ \\
Skewed (80/10/5/5) & 0.456ms & 0.389ms & 1.17$\times$ \\
Extreme (95/2/2/1) & 0.521ms & 0.372ms & 1.40$\times$ \\
\bottomrule
\end{tabular}
\end{table}

The benefit is largest for skewed distributions where some adapters have very few tokens. Fixed 64$\times$64 tiles waste computation on small groups; adaptive 16$\times$32 tiles reduce this overhead.

\paragraph{Implementation.}
We implement adaptive tiling in Triton by dispatching to different kernel configurations based on histogram counts. Each configuration is pre-compiled; the histogram determines which to invoke. This adds minimal overhead ($<$1$\mu$s for histogram analysis) while providing significant benefits for non-uniform workloads.

\subsection{Kernel Performance Analysis}

Scalar kernels (S-LoRA's BGMV) achieve superior small-batch performance through CUDA-specific optimizations: \texttt{cuda::memcpy\_async} for pipelined memory operations, \texttt{\_\_shfl\_down\_sync} for warp-level reductions, and hand-tuned thread configurations. These optimizations are most effective in the memory-bound regime. Our tensor-core implementation with CUDA graph capture eliminates kernel launch overhead, achieving near-constant latency across batch sizes. At batch sizes $\geq$160, our kernel outperforms BGMV.

\begin{figure}[h]
\centering
\begin{tikzpicture}
\begin{axis}[
    width=0.9\columnwidth,
    height=6cm,
    xlabel={Batch Size},
    ylabel={Latency (ms)},
    xmode=log,
    log basis x={2},
    xtick={64, 128, 256, 512, 1024},
    xticklabels={64, 128, 256, 512, 1024},
    ymin=0.010, ymax=0.060,
    ytick={0.01, 0.02, 0.03, 0.04, 0.05, 0.06},
    grid=major,
    grid style={dashed, gray!40},
    legend style={at={(0.03,0.97)}, anchor=north west, font=\small},
    tick label style={font=\small},
    label style={font=\small},
]

\addplot[thick, mark=square*, red!80!black, mark size=2.5pt] coordinates {
    (64, 0.0132)
    (128, 0.0156)
    (160, 0.0170)
    (256, 0.0209)
    (384, 0.0263)
    (512, 0.0318)
    (1024, 0.0535)
};
\addlegendentry{BGMV (Scalar)}

\addplot[thick, mark=*, blue!80!black, mark size=2.5pt] coordinates {
    (64, 0.0157)
    (128, 0.0158)
    (160, 0.0158)
    (256, 0.0159)
    (384, 0.0160)
    (512, 0.0160)
    (1024, 0.0199)
};
\addlegendentry{Ours (Tensor Core)}

\draw[gray, dashed, thick] (axis cs:160,0.010) -- (axis cs:160,0.055);
\node[font=\scriptsize, gray, anchor=south, rotate=90] at (axis cs:160,0.035) {Crossover};

\node[font=\scriptsize, blue!60!black, anchor=west] at (axis cs:550,0.032) {1.98$\times$};
\node[font=\scriptsize, blue!60!black, anchor=west] at (axis cs:1100,0.040) {2.69$\times$};

\end{axis}
\end{tikzpicture}
\caption{Kernel comparison: Tensor Core (with CUDA graph) vs BGMV scalar kernel. Our kernel achieves near-constant latency ($\sim$0.016ms) while BGMV scales linearly. Crossover at batch 160; at production sizes (512--1024), we achieve 1.98--2.69$\times$ speedup. Benchmarked on H100 NVL with $d$=4096, $r$=64, 8 adapters.}
\label{fig:kernel_benchmark}
\end{figure}
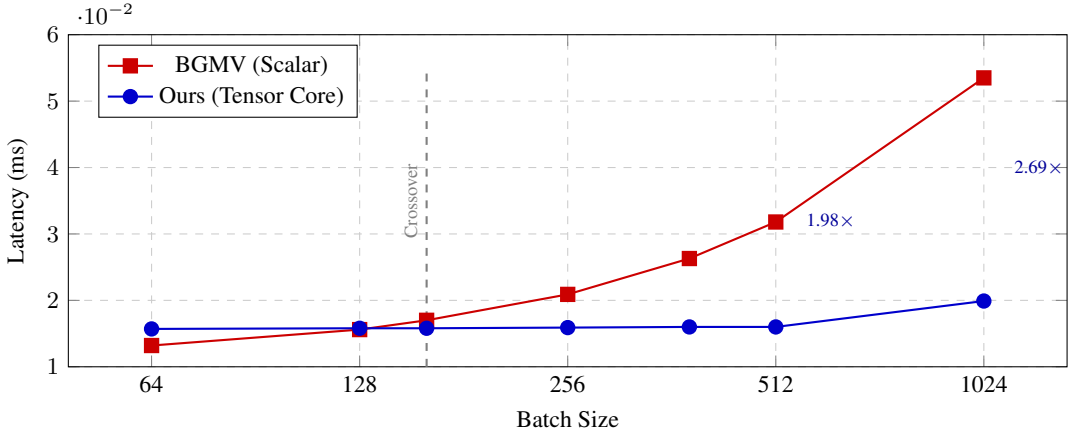

Figure~\ref{fig:kernel_benchmark} shows the kernel comparison. The crossover occurs at batch 160, where our tensor-core kernel begins to outperform BGMV. At production batch sizes (512--1024), we achieve 1.98--2.69$\times$ speedup. CUDA graph capture eliminates Python and kernel launch overhead, making our kernel's latency nearly constant ($\sim$0.016ms) while BGMV scales linearly with batch size.

\section{Dispatch Infrastructure Details}
\label{app:dispatch}

The unified dispatch infrastructure described in \S\ref{sec:dispatch} enables code reuse between MoE and multi-adapter systems. We detail the correspondence and implementation.

\subsection{MoE and Multi-Adapter Correspondence}

Table~\ref{tab:moe_adapter} provides a complete mapping between MoE and multi-adapter serving concepts. Post-routing computation is identical---only the routing decision mechanism differs.

\begin{table}[h]
\centering
\caption{Complete correspondence between MoE and multi-adapter serving.}
\label{tab:moe_adapter}
\begin{tabular}{lll}
\toprule
\textbf{Concept} & \textbf{MoE} & \textbf{Multi-Adapter} \\
\midrule
Routing function & $g_\theta(x) = \text{softmax}(Wx + b)$ & $j: [B] \to [K]$ or $\cR_\text{vocab}(v)$ \\
Routing cost & $\cO(E \cdot d)$ (learned) & $\cO(1)$ (deterministic) \\
Routing decision & Per-token, learned & Per-sequence or per-token \\
Targets & Experts $\{1, \ldots, E\}$ & Adapters $\{1, \ldots, K\}$ \\
Weight shape & $(E, d_\text{in}, d_\text{out})$ & $(K, d, r)$ for $A$; $(K, r, d)$ for $B$ \\
Capacity & Expert capacity $C$ & Adapter memory budget \\
Load balancing & Auxiliary losses & Request-level shaping \\
\midrule
\multicolumn{3}{c}{\textit{Identical Infrastructure}} \\
\midrule
Histogram & $h[e] = |\{t : r(t) = e\}|$ & $h[k] = |\{t : r(t) = k\}|$ \\
Pointer arrays & $\mathbf{xs}[e], \mathbf{ys}[e]$ & $\mathbf{xs}[k], \mathbf{ys}[k]$ \\
Dispatch kernel & Target-agnostic & Target-agnostic \\
Compute kernel & Grouped GEMM & Grouped GEMM \\
\bottomrule
\end{tabular}
\end{table}

\subsection{Unified Dispatch Algorithm}

Algorithm~\ref{alg:unified_dispatch} shows the complete dispatch kernel, which is identical for MoE and multi-adapter serving.

\begin{algorithm}[h]
\caption{Unified Dispatch Kernel (Target-Agnostic)}
\label{alg:unified_dispatch}
\begin{algorithmic}[1]
\Require Routing decisions $r \in [C]^N$, input pointer $x_\text{ptr}$, output pointer $y_\text{ptr}$, strides
\Ensure Histogram $h \in \ZZ^C$, pointer arrays $\mathbf{xs}, \mathbf{ys}$
\State $h \gets \mathbf{0}_C$
\For{token $t = 1, \ldots, N$ \textbf{in parallel}}
    \State $c \gets r[t]$ \Comment{Target index (expert or adapter)}
    \State $\text{pos} \gets \text{atomicAdd}(h[c], 1)$ \Comment{Allocate position}
    \State $\mathbf{xs}[c, \text{pos}] \gets x_\text{ptr} + t \cdot \text{stride}_x$ \Comment{Input pointer}
    \State $\mathbf{ys}[c, \text{pos}] \gets y_\text{ptr} + t \cdot \text{stride}_y$ \Comment{Output pointer}
\EndFor
\end{algorithmic}
\end{algorithm}

The histogram $h$ is then used by the compute kernel for adaptive tiling: targets with few tokens use smaller tiles, while targets with many tokens use larger tiles for better throughput.

\subsection{Implications for System Design}

The dispatch unification has several practical implications:

\begin{enumerate}
    \item \textbf{Code reuse}: Optimizations developed for MoE dispatch (e.g., MegaBlocks~\citep{gale2023megablocks}) apply directly to multi-adapter serving.
    \item \textbf{Hybrid systems}: A single dispatch infrastructure can route some tokens to experts (learned routing) and others to adapters (deterministic routing).
    \item \textbf{Compositional routing}: The product space $\cA \times \cM$ is handled by computing composite indices $c = a \cdot |\cM| + m$ before dispatch.
\end{enumerate}

\section{Additional Results}
\label{app:results}

\subsection{Latency Variance}
\label{app:variance}

\begin{table}[h]
\centering
\caption{Latency variance comparison (CV = std/mean).}
\begin{tabular}{lcccc}
\toprule
\textbf{Batch} & \textbf{Eager CV} & \textbf{Graph CV} & \textbf{Paging CV} \\
\midrule
32 & 0.52 & 0.03 & 0.63 \\
64 & 0.68 & 0.33 & 0.32 \\
128 & 0.34 & 0.12 & 0.12 \\
256 & 0.01 & 0.07 & 0.18 \\
\bottomrule
\end{tabular}
\end{table}

\section{Additional MoLoRA Results}
\label{app:molora}

This appendix contains additional MoLoRA experimental results that support the main text findings.

\subsection{Load Balancing Details}

Following Switch Transformer~\citep{fedus2022switch}, we add an auxiliary loss to encourage uniform adapter utilization:
\begin{equation}
    \cL_\text{aux} = K \cdot \sum_{j=1}^{K} f_j \cdot p_j
\end{equation}
where $f_j$ is the fraction of tokens routed to adapter $j$ (based on top-1 selection) and $p_j$ is the mean routing probability for adapter $j$. This loss penalizes configurations where high-probability adapters also receive many tokens, encouraging balanced utilization.

\subsection{Use Case Taxonomy}

\begin{table}[h]
\centering
\caption{MoLoRA use case taxonomy. Vocabulary routing requires modality-encoding vocabulary structure. MoLoRA handles all scenarios through learned routing.}
\label{tab:molora_usecases_appendix}
\begin{tabular}{lcc}
\toprule
\textbf{Scenario} & \textbf{Vocab Routing} & \textbf{MoLoRA} \\
\midrule
Chameleon-style (disjoint vocab) & \checkmark & \checkmark \\
Encoder-based multimodal (LLaVA, Flamingo) & \ding{55} & \checkmark \\
Semantic specialization (code/math/prose) & \ding{55} & \checkmark \\
Sub-modality granularity (photo/diagram/chart) & \ding{55} & \checkmark \\
Multi-attribute routing (modality $\times$ domain) & \ding{55} & \checkmark \\
\bottomrule
\end{tabular}
\end{table}

\subsection{Synthetic Multimodal Task}

We evaluate MoLoRA on a synthetic multimodal task where each of 3 modalities has a distinct optimal transformation. The task tests whether MoLoRA can learn modality-specialized routing without access to modality labels.

\paragraph{Setup.} We use $d=256$, rank $r=16$, and compare four approaches: (1) \textbf{Single}: one adapter for all tokens; (2) \textbf{Fixed (Oracle)}: ground-truth modality labels determine routing; (3) \textbf{MoLoRA}: learned routing with 4 adapters, top-$k=2$; (4) \textbf{MoLoRA-L}: 8 adapters, top-$k=3$.

\begin{table}[h]
\centering
\caption{MoLoRA training results on synthetic multimodal task. MoLoRA learns routing without labels, achieving 35\% of oracle improvement.}
\label{tab:molora_appendix}
\begin{tabular}{lcccc}
\toprule
\textbf{Model} & \textbf{Adapters} & \textbf{top-$k$} & \textbf{Train Loss} & \textbf{Val Loss} \\
\midrule
Single Adapter & 1 & -- & 3.008 & 3.435 \\
Fixed (Oracle) & 3 & 1 & 2.588 & 3.819 \\
MoLoRA & 4 & 2 & 2.861 & 3.555 \\
MoLoRA-L & 8 & 3 & 2.783 & 3.599 \\
\bottomrule
\end{tabular}
\end{table}

Table~\ref{tab:molora_appendix} shows that MoLoRA reduces training loss by 4.9\% over single-adapter (2.861 vs 3.008), achieving 35\% of the oracle improvement. The oracle (fixed routing with known labels) achieves 14\% improvement, demonstrating the value of modality-specialized adapters. MoLoRA approaches this without access to labels.

\subsection{Emergent Modality Discovery}

\begin{table}[h]
\centering
\caption{MoLoRA discovers modality structure without supervision. ARI (Adjusted Rand Index) and NMI (Normalized Mutual Information) measure alignment between learned routing and true modality labels. Higher is better; 1.0 indicates perfect clustering.}
\label{tab:modality_discovery_appendix}
\begin{tabular}{lccc}
\toprule
\textbf{Epoch} & \textbf{ARI} & \textbf{NMI} & \textbf{Routing Entropy} \\
\midrule
0 (random) & 0.27 & 0.37 & 1.35 \\
20 & 0.72 & 0.86 & 0.09 \\
99 (converged) & \textbf{0.71} & \textbf{0.84} & 0.18 \\
\bottomrule
\end{tabular}
\end{table}

Table~\ref{tab:modality_discovery_appendix} reveals a surprising finding: \textbf{MoLoRA's router automatically discovers modality structure despite never receiving modality labels during training}. Starting from random routing (ARI=0.27), the router converges to near-perfect modality clustering (ARI=0.71, NMI=0.84). The confusion matrix shows each adapter specializes to specific modalities:

\begin{center}
\small
\begin{tabular}{c|cccc}
 & A0 & A1 & A2 & A3 \\
\hline
Modality 0 & \textbf{1.00} & 0.00 & 0.00 & 0.00 \\
Modality 1 & 0.00 & 0.00 & 0.00 & \textbf{1.00} \\
Modality 2 & 0.00 & 0.08 & \textbf{0.92} & 0.00 \\
Modality 3 & \textbf{1.00} & 0.00 & 0.00 & 0.00 \\
\end{tabular}
\end{center}

This emergent behavior demonstrates that the optimal routing strategy---grouping tokens by modality---is \emph{learnable from data alone}. When vocabulary structure encodes modality (as in Chameleon), deterministic routing is sufficient. When it does not, MoLoRA provides a path to modality-aware adaptation.

\subsection{Semantic Domain Routing}

To validate that MoLoRA handles semantic specialization when vocabulary provides no signal, we train a router on synthetic code/math/prose embeddings. Critically, all domains share the same vocabulary range---the word ``function'' has identical token IDs whether appearing in code (\texttt{def function(x):}), mathematics (``continuous function $f(x)$''), or prose (``the function of education''). Vocabulary routing \emph{cannot} distinguish these cases.

\begin{table}[h]
\centering
\caption{MoLoRA semantic routing: perfect domain specialization (ARI=1.0) despite shared vocabulary. Each domain routes exclusively to a single adapter.}
\label{tab:semantic_routing_appendix}
\begin{tabular}{lccc}
\toprule
\textbf{Domain} & \textbf{Adapter 0} & \textbf{Adapter 1} & \textbf{Adapter 2} \\
\midrule
Code & 0.0\% & \textbf{100\%} & 0.0\% \\
Math & 0.0\% & 0.0\% & \textbf{100\%} \\
Prose & \textbf{100\%} & 0.0\% & 0.0\% \\
\bottomrule
\end{tabular}
\end{table}

Table~\ref{tab:semantic_routing_appendix} shows that MoLoRA achieves \textbf{perfect specialization} (ARI=1.0, NMI=1.0) within 10 epochs. The router learns to distinguish domains from embedding context alone: code embeddings activate syntax-related dimensions, math embeddings activate symbolic/logical dimensions, and prose embeddings activate semantic/narrative dimensions. Despite these patterns being invisible to vocabulary-based routing, the learned router identifies them immediately.

This result validates the core claim: MoLoRA enables adapter specialization along \emph{any} dimension---modality, domain, style, task---without requiring that dimension to be encoded in vocabulary structure.

\subsection{Inference Overhead}

\begin{table}[h]
\centering
\caption{MoLoRA inference overhead at $d=4096$, batch=32, seq=128.}
\label{tab:molora_inference_appendix}
\begin{tabular}{lccc}
\toprule
\textbf{Model} & \textbf{Latency (ms)} & \textbf{Throughput (M tok/s)} & \textbf{Overhead} \\
\midrule
Single Adapter & 0.082 & 49.7 & 1.0$\times$ \\
Fixed Routing & 0.894 & 4.6 & 10.9$\times$ \\
MoLoRA $k=1$ & 0.995 & 4.1 & 12.1$\times$ \\
MoLoRA $k=2$ & 2.267 & 1.8 & 27.6$\times$ \\
MoLoRA $k=4$ & 4.001 & 1.0 & 48.8$\times$ \\
\bottomrule
\end{tabular}
\end{table}

Table~\ref{tab:molora_inference_appendix} shows the quality--latency trade-off. MoLoRA with $k=1$ has similar overhead to fixed routing (the router MLP is negligible). Higher $k$ improves quality but increases latency linearly, as each selected adapter requires a full LoRA forward pass. For quality-sensitive applications, MoLoRA $k=2$ provides a practical operating point.

\subsection{Routing Analysis}

To validate MoLoRA on production models, we train a router on Qwen3-4B~\citep{qwen2025qwen3} embeddings. We construct a multi-domain dataset with code, math, creative writing, and technical content, then train a lightweight router to classify content type.

\paragraph{Setup.} We extract embeddings from Qwen3-4B (hidden size 2560), normalize them for numerical stability, and train a 2-layer router MLP (2560 $\to$ 128 $\to$ 4) with cross-entropy loss. The router achieves 100\% classification accuracy within 100 epochs.

\paragraph{Per-Token Specialization.} Table~\ref{tab:molora_qwen_appendix} shows that the router learns strong per-token specialization: code tokens route 98.6\% to Adapter 0, math tokens 96.6\% to Adapter 1, creative tokens 98.8\% to Adapter 2, and technical tokens 99.2\% to Adapter 3. This demonstrates that MoLoRA can learn domain-specific routing from real LLM embeddings.

\begin{table}[h]
\centering
\caption{MoLoRA routing specialization on Qwen3-4B. Each content type routes predominantly to a single adapter, achieving near-perfect specialization.}
\label{tab:molora_qwen_appendix}
\begin{tabular}{lcccc}
\toprule
\textbf{Content Type} & \textbf{Adapter 0} & \textbf{Adapter 1} & \textbf{Adapter 2} & \textbf{Adapter 3} \\
\midrule
Code & \textbf{98.6\%} & 0.0\% & 0.0\% & 1.4\% \\
Math & 0.0\% & \textbf{96.6\%} & 0.5\% & 3.0\% \\
Creative & 0.0\% & 0.0\% & \textbf{98.8\%} & 1.2\% \\
Technical & 0.3\% & 0.3\% & 0.2\% & \textbf{99.2\%} \\
\bottomrule
\end{tabular}
\end{table}

\paragraph{Mixed-Content Handling.} On multi-domain inputs (e.g., code with mathematical comments), the router correctly assigns different tokens to different adapters within the same sequence. For example, in ``\texttt{def integrate(f, a, b):} \texttt{'''Numerical integration...}''', function definition tokens route to the Code adapter while ``Numerical'' routes to Math and ``integration'' to Technical. This per-token granularity is precisely what enables efficient mixed-content serving.

\end{document}